\documentclass[letterpaper,10pt,conference]{ieeeconf}

\IEEEoverridecommandlockouts
\usepackage{amsmath,amssymb,amsfonts}
\usepackage{graphicx}
\usepackage{booktabs}
\usepackage{cite}
\usepackage{microtype}

\newtheorem{theorem}{Theorem}

\newcommand{\R}{\mathbb{R}}
\newcommand{\norm}[1]{\left\lVert#1\right\rVert}
\newcommand{\sat}{\operatorname{sat}}
\newcommand{\one}{\mathbf{1}}
\newcommand{\T}{^{\mathsf T}}

\newcommand{\SimDuration}{12}
\newcommand{\SimStepMs}{2.0}
\newcommand{\SimAcquisitionTime}{1.5}
\newcommand{\SimKZero}{0.50}
\newcommand{\SimKAmplitude}{0.05}
\newcommand{\SimKFrequency}{0.40}
\newcommand{\SimBeta}{2}
\newcommand{\SimKOne}{\operatorname{diag}(24,12)}

\newcommand{\SimKTwo}{24}
\newcommand{\SimActivationGain}{0.15}
\newcommand{\SimTorqueLimitOne}{35}
\newcommand{\SimTorqueLimitTwo}{12}
\newcommand{\SimInitialErrorOne}{-0.70}
\newcommand{\SimInitialErrorTwo}{0.35}
\newcommand{\SimInitialRatio}{1.423}
\newcommand{\SimRBFCount}{10}
\newcommand{\SimSigmaA}{0.02}
\newcommand{\SimSigmaC}{0.03}
\newcommand{\SimChiA}{0.003}
\newcommand{\SimChiC}{0.30}
\newcommand{\SimKappa}{0.003}
\newcommand{\SimDiscount}{1}
\newcommand{\SimFilterTime}{0.08}
\newcommand{\SimRBFWidth}{1.2}
\newcommand{\SimCostError}{1}
\newcommand{\SimCostVelocity}{0.05}
\newcommand{\SimCostInput}{0.002}
\newcommand{\SimCTFrequency}{8}
\newcommand{\SimCTDamping}{1}
\newcommand{\SimMassOne}{2}
\newcommand{\SimMassTwo}{1}
\newcommand{\SimLinkOne}{1}
\newcommand{\SimCOMOne}{0.5}
\newcommand{\SimCOMTwo}{0.375}
\newcommand{\SimInertiaOne}{0.2}
\newcommand{\SimInertiaTwo}{0.2}

\newcommand{\SimProposedRho}{0.053}
\newcommand{\SimProposedEntry}{0.210}
\newcommand{\SimProposedSaturation}{4.6}

\newcommand{\SimStressCaseCount}{12}
\newcommand{\SimMassIncrementMin}{0.0}
\newcommand{\SimMassIncrementMax}{0.5}
\newcommand{\SimFrictionMin}{0.8}
\newcommand{\SimFrictionMax}{1.3}
\newcommand{\SimDisturbanceMin}{0.5}
\newcommand{\SimDisturbanceMax}{1.5}
\newcommand{\SimSpeedMin}{0.8}
\newcommand{\SimSpeedMax}{1.18}
\newcommand{\SimNoiseMax}{0.002}
\newcommand{\SimulationScenarioRows}{%
1 & 0.00 & 0.80 & 0.50 & 0.0000 & 0.00 & 0.80 \\
2 & 0.25 & 0.95 & 1.25 & 0.0008 & 0.35 & 1.06 \\
3 & 0.50 & 1.10 & 0.97 & 0.0016 & 0.70 & 0.89 \\
4 & 0.00 & 1.24 & 0.68 & 0.0004 & 1.05 & 1.15 \\
5 & 0.25 & 0.86 & 1.44 & 0.0012 & 1.40 & 0.99 \\
6 & 0.50 & 1.01 & 1.15 & 0.0020 & 1.75 & 0.82 \\
7 & 0.00 & 1.15 & 0.85 & 0.0007 & 2.09 & 1.08 \\
8 & 0.25 & 1.30 & 0.56 & 0.0015 & 2.44 & 0.92 \\
9 & 0.50 & 0.92 & 1.32 & 0.0002 & 2.79 & 1.18 \\
10 & 0.00 & 1.06 & 1.03 & 0.0011 & 3.14 & 1.01 \\
11 & 0.25 & 1.21 & 0.74 & 0.0019 & 3.49 & 0.85 \\
12 & 0.50 & 0.83 & 1.50 & 0.0006 & 3.84 & 1.11 \\
}
\newcommand{\SimProposedStressSuccess}{12}
\newcommand{\SimProposedStressMedianRMSE}{0.0250}

\newcommand{\SimProposedStressWorstRho}{0.075}
\newcommand{\SimProposedStressMedianEntry}{0.258}
\newcommand{\SimFrozenStressSuccess}{12}
\newcommand{\SimFrozenStressMedianRMSE}{0.0456}

\newcommand{\SimFrozenStressWorstRho}{0.131}
\newcommand{\SimFrozenStressMedianEntry}{0.258}
\newcommand{\SimCTStressSuccess}{12}
\newcommand{\SimCTStressMedianRMSE}{0.0498}

\newcommand{\SimCTStressWorstRho}{0.243}
\newcommand{\SimCTStressMedianEntry}{0.421}
\newcommand{\SimStressAdaptiveReduction}{45.1}
\newcommand{\SimMapSuccess}{36}
\newcommand{\SimMapTotal}{49}
\newcommand{\SimMapInfeasibleSuccess}{29}
\newcommand{\SimMapInfeasibleTotal}{42}
\newcommand{\SimFineStepMs}{1.0}
\newcommand{\SimCoarseFineRMSEChange}{0.034}

\newcommand{\ExpSampleCount}{2880}
\newcommand{\ExpDuration}{120}
\newcommand{\ExpLogPeriodMs}{41.667}
\newcommand{\ExpSteadyStart}{20}
\newcommand{\ExpWindowLength}{20}
\newcommand{\ExpWindowCount}{6}
\newcommand{\ExpFullMAE}{0.957}
\newcommand{\ExpFullRMSE}{1.154}
\newcommand{\ExpFullPercentile}{2.141}
\newcommand{\ExpFullMax}{3.520}
\newcommand{\ExpSteadyMAE}{0.940}
\newcommand{\ExpSteadyRMSE}{1.125}
\newcommand{\ExpSteadyPercentile}{2.089}
\newcommand{\ExpSteadyMax}{2.948}
\newcommand{\ExpWindowRMSEMin}{1.032}
\newcommand{\ExpWindowRMSEMax}{1.293}
\newcommand{\ExpGuardCount}{0}
\newcommand{\ExpSaturationCount}{0}

\newif\ifanonymous
\anonymousfalse

\title{\LARGE \bf Learning-Based Progressive Barrier Control for Robot Manipulators with Initial Errors Outside Prescribed Tracking Bounds}

\ifanonymous
\author{Anonymous Authors}
\else
\author{%
Hamed Rahimi Nohooji$^{1}$, Danial Zafaranchizadeh Moghaddam$^{2}$,
Abolfazl Zaraki$^{2}$, and Holger Voos$^{1}$%
\thanks{$^{1}$Automation and Robotics Research Group, Interdisciplinary
Centre for Security, Reliability and Trust (SnT), University of Luxembourg,
Luxembourg.}%
\thanks{$^{2}$Department of Computer Science, School of Physics,
Engineering and Computer Science, University of Hertfordshire,
Hatfield, United Kingdom.}%
}
\fi

\begin{document}
\maketitle
\thispagestyle{empty}
\pagestyle{empty}

\begin{abstract}
Robot manipulators may start a new task with a tracking error larger than the prescribed tolerance. Conventional barrier controllers generally require the initial error to lie
within this tolerance, which prevents their direct use under such conditions. This paper develops a progressive barrier controller that gradually contracts an initially relaxed error bound to the required value within a prescribed time. The robot can therefore start outside the final bound, while the direct joint-position error is constrained after the transition. The closed-form control law combines progressive barrier feedback with online actor--critic compensation for uncertain dynamics and disturbances. A Lyapunov analysis establishes bounded closed-loop signals and gives sufficient conditions for satisfaction of the final tracking bound. Two-link simulations consider large initial errors, torque saturation, dynamic variations, disturbances, and measurement errors. The adaptive term reduces the median root-mean-square tracking error by
\SimStressAdaptiveReduction\% compared with the same progressive barrier
controller without learning. The simulations also show how the selected transition time and available torque determine the range of initial errors that can be handled. Finally, an experiment on a Niryo Ned3 Pro illustrates tracking performance using measured position and motor-current data.
\end{abstract}


\section{Introduction}

Robot manipulators often move between tasks, reference trajectories, and
operating modes. At the start of a new motion, the joint tracking error can be
larger than the tolerance required during task execution, even though the
robot remains within its joint limits and the planned motion is collision
free. This situation is common after a reference switch, replanning event, or
recovery action. The tracking tolerance therefore has a different timing
requirement from a physical safety constraint: joint and collision limits must
hold throughout the motion, whereas a tight tracking bound may only be needed
after a short transition. The control problem is to bring an initially
out-of-bound tracking error into the required tolerance within a specified
time, despite uncertain dynamics and finite actuator torque, and to maintain
that tolerance thereafter.

There are several constrained-control frameworks used in robotics. Control barrier functions (CBFs)
are well suited to joint limits, collision avoidance, and other safe sets that
must remain invariant from the initial time
\cite{ames2017cbf,rauscher2016robot}. Their implementation commonly uses a
quadratic program and must account for feasibility, sampling, and model
uncertainty \cite{breeden2022sampled,garg2024advances,lopez2021adaptive}.
Model predictive control (MPC) can impose state and input constraints over a
prediction horizon \cite{mayne2000mpc}, but it requires a prediction model and
the repeated solution of a constrained optimization problem. Prescribed-performance control (PPC) shapes transient and steady-state accuracy through a
time-varying error envelope \cite{bechlioulis2008prescribed}; in its standard
form, the error transformation must be defined inside that envelope at the
initial time. Global- and deferred-performance designs relax this condition
\cite{song2018deferred,zhao2019feasibility,zhao2022deferred,kong2023global},
including for manipulators with input nonlinearities
\cite{zhang2023deferredrobot}. These are useful alternatives. Here we seek a
closed-form low-level controller for a task-level tracking tolerance that is
allowed to become active after a prescribed transition, without solving an
optimization problem online.

Barrier Lyapunov functions (BLFs) offer a closed-form alternative for
constrained nonlinear control. A BLF grows without bound as its argument
approaches the boundary, and boundedness of the closed-loop Lyapunov function
then prevents the boundary from being reached \cite{tee2009barrier}. This
principle has been used for uncertain robot manipulators with joint-space and
time-varying constraints \cite{tang2016adaptive,rahimi2018joint}. However, the
standard  BLF is finite only when its initial argument already lies
inside the prescribed set. Consequently, a controller designed directly for a
tight final tracking tolerance cannot be initialized when a new task starts
with a larger error. Enlarging the bound for the entire motion avoids this
problem, but also weakens the accuracy requirement after the transient.

Uncertain inertial parameters, friction, payload changes, and external loads
introduce an additional tracking error. Learning-based compensation has been
used to improve the tracking of uncertain Lagrangian systems
\cite{helwa2019robust}, and its use in robotic control has been studied more
broadly in \cite{brunke2022safe}. Actor--critic methods provide one mechanism
for adapting a control contribution from tracking and control-effort data
\cite{vamvoudakis2010actor,nohooji2024actor}. In this paper, the actor--critic
structure is used as an online torque compensator.  The stability result is stated in terms of the
combined uncertainty that remains after compensation, including disturbances,
sampling effects, and the difference between commanded and applied torque.

To address these points, we introduce a progressive barrier controller for
robot joint tracking. A raised-cosine factor smoothly scales the direct
joint-position error from zero to its full value over a prescribed transition
time. The barrier is therefore well defined at initialization for
any finite tracking error. After the transition, the scaling is exactly one,
so the variable constrained by the barrier is the direct joint-position error.
The resulting control law is closed form and does not require an online
optimizer. Actor--critic torque compensation is included for uncertain
dynamics and disturbances, and the analysis is written for the torque applied
to the robot rather than only for the unconstrained controller command.

The main contributions are as follows:
\begin{itemize}
\item A progressive barrier is formulated directly from the
joint-position error. It removes the initial-domain restriction associated
with the final tracking tolerance and recovers that tolerance at a prescribed
time.
\item The barrier feedback is combined with online actor--critic compensation,
and a residual-based Lyapunov condition is derived for bounded closed-loop
signals and satisfaction of the post-transition tracking bound. The condition
does not require convergence of the learned weights or an optimal-policy
claim, and it explicitly uses the applied torque.
\item The controller is evaluated on a manipulator subject to joint
torque limits, initial errors outside the final tolerance, dynamic variations,
disturbances, and position-measurement errors. A numerical map relates the
initial error and transition time to feasibility under the available torque.
Measured position and motor-current data from a Niryo Ned3 Pro provide an
additional single-joint hardware evaluation.
\end{itemize}

Figure~\ref{fig:overview} summarizes the progressive transition and shows the
robot used for the single-joint experiment.

\begin{figure*}[t]
    \centering
    \includegraphics[width=0.97\textwidth]{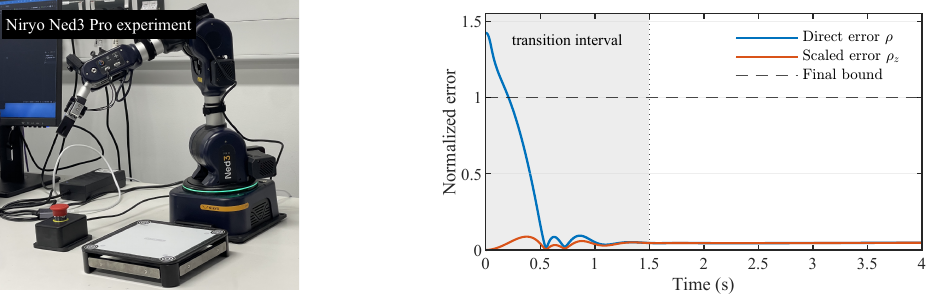}
    \caption{Progressive imposition of the joint tracking bound. The photograph
    at left shows the Niryo Ned3 Pro used in the single-joint experiment. The
    plot at right shows the nominal two-link simulation. The direct tracking
    error may initially lie outside the final bound, whereas the scaled error
    starts inside the barrier domain. The two errors coincide after the
    prescribed transition time.}
    \label{fig:overview}
\end{figure*}

\section{Problem Formulation}

Consider an $n$-joint rigid manipulator \cite{spong2020robot}
\begin{equation}
 M(q)\ddot q+C(q,\dot q)\dot q+g(q)=u_{\rm ap}+d,
 \label{eq:robot}
\end{equation}
where $q,\dot q,\ddot q\in\R^n$ denote the joint positions, velocities, and
accelerations, respectively. The matrix
$M(q)\in\R^{n\times n}$ is the inertia matrix,
$C(q,\dot q)\in\R^{n\times n}$ contains the Coriolis and centrifugal
terms, and $g(q)\in\R^n$ is the gravity vector. The vector $d\in\R^n$
collects additive joint-space disturbances, including external torques and
unmodeled effects. The signal $u_{\rm ap}\in\R^n$ denotes the torque actually
applied to the robot; it may differ from the controller command
$u_{\rm cmd}$ because of actuator saturation and interface effects. The
inertia matrix satisfies $M=M\T\succ0$ and
$\underline m I\preceq M(q)\preceq\bar m I$ on the operating region.
Moreover, $\dot M-2C$ is skew-symmetric.

Let the desired trajectory $q_d(t)$ be twice continuously differentiable and define the direct
tracking error
\begin{equation}
 e=q-q_d. \label{eq:error}
\end{equation}
The required tracking tolerance is specified by the smooth scalar envelope
$k(t)\ge k_{\min}>0$.  The control objective is
\begin{equation}
 \norm{e(t)}<k(t),\qquad t\ge T_c.                 \label{eq:objective}
\end{equation}
No initial condition of the form $\norm{e(0)}<k(0)$ is required. The bound
in (\ref{eq:objective}) is imposed only for $t\geq T_c$; before $T_c$, the
direct tracking error may lie outside $k(t)$. Because the same scalar envelope
is applied to all joints, this condition is conservative. If every desired
joint position remains at least $k(t)$ away from its physical limits,
satisfaction of (\ref{eq:objective}) also guarantees that the joint limits are
respected for $t\geq T_c$.

To introduce the final tracking bound smoothly, we define the raised-cosine
transition function
\begin{equation}
 \gamma(t)=
 \begin{cases}
 \tfrac12[1-\cos(\pi t/T_c)],&0\le t<T_c,\\
 1,&t\ge T_c,
 \end{cases}                                      \label{eq:gamma}
\end{equation}
and the transformed direct error
\begin{equation}
 z=\gamma e.                   \label{eq:zD}
\end{equation}
Thus $z(0)=0$ for every finite $e(0)$, while $z=e$ after $T_c$.  For
$\gamma>0$, $\norm{z}<k$ is exactly the temporary condition
$\norm{e}<k/\gamma$; this is a time-varying relaxation, not a new barrier
geometry.

\section{Progressive Barrier Control with Adaptive Compensation}

Using the scaled error $z$, we define the progressive barrier Lyapunov
function
\begin{equation}
B(z,t)=\frac{1}{2\beta}\log\left(\frac{k^2}{D}\right),
\qquad \beta>0,                                      \label{eq:blf}
\end{equation}
where $D:=k^2-z\T z$. The condition $D>0$ is
equivalent to $\norm{z}<k$. The constant $\beta$ scales the barrier function
and its corresponding feedback term but does not change this admissible
domain.
  Define the virtual velocity and
velocity error as
\begin{align}
 \alpha={}&\dot q_d+\frac{\dot k}{k}e-K_1e
 -\frac{a\dot\gamma^{\,2}e\norm{e}^2}{\beta D},   \label{eq:alpha}\\
 s={}&\dot q-\alpha,                               \label{eq:s}
\end{align}
where $K_1\succ0$ and $a>0$.  The commanded input is
\begin{equation}
 u_{\rm cmd}=T_a-K_2s-\frac{\gamma z}{\beta D},
 \qquad K_2\succ0.                                \label{eq:ucmd}
\end{equation}
Since $\gamma=1$ for $t\geq T_c$, the barrier acts directly on the
joint-position error after the prescribed transition time.

\subsection{Actor--Critic Compensation}

Let $\phi_a(x_a)\in\R^{m_a}$ and $\phi_c(x_c)\in\R^{m_c}$ be radial basis function (RBF) feature maps with fixed centers and widths.  With actor matrix $\hat A\in\R^{m_a\times n}$ and critic
weights $\hat c\in\R^{m_c}$,
\begin{equation}
 T_a=\hat A\T\phi_a\in\R^n,\qquad
 \widehat{\mathcal C}=\hat c\T\phi_c\in\R.       \label{eq:outputs}
\end{equation}
The implementation uses
\begin{align}
 \dot{\hat A}={}&-\sigma_a\phi_a
 \left(T_a+\frac{s}{\sigma_a}
 +\kappa\widehat{\mathcal C}\one_n\right)\T
 -\sigma_a\chi_a\hat A,                          \label{eq:actor}\\
 \xi={}&\ell+\hat c\T\pi_c,\qquad
 \dot{\hat c}=-\sigma_c\xi\pi_c-\sigma_c\chi_c\hat c,
                                                        \label{eq:critic}
\end{align}
where $\ell=[e\T\;s\T]Q[e\T\;s\T]\T+u_{\rm ap}\T Ru_{\rm ap}$,
with $Q\in\R^{2n\times2n}$, $Q\succeq0$, and
$R\in\R^{n\times n}$, $R\succ0$.  The critic approximates the discounted cost
\begin{equation}
 \mathcal C(t)=\int_t^\infty e^{-(\tau-t)/\nu}\ell(\tau)\,d\tau,
 \qquad \nu>0.                                      \label{eq:cost}
\end{equation}
The gains $\sigma_a,\sigma_c,\chi_a,\chi_c$, and $\kappa$ are positive.
Here $\one_n\in\R^n$ is the vector of ones.  To avoid numerical
differentiation, $\pi_c$ uses a stable feature filter:
\begin{equation}
 \tau_f\dot{\bar\phi}_c+\bar\phi_c=\phi_c,\qquad
 \pi_c=(\phi_c-\bar\phi_c)/\tau_f-\phi_c/\nu .    \label{eq:filter}
\end{equation}
Here $\tau_f>0$ is the feature-filter time constant.  The filter acts only on
the critic feature vector and supplies a filtered approximation of
$\dot\phi_c$; it does not alter the tracking error used in the barrier.
The update laws satisfy
\begin{align}
 \frac{1}{2\sigma_c}\frac{d}{dt}\norm{\hat c}^2
 &\le \frac{\ell^2}{4}-\chi_c\norm{\hat c}^2,       \label{eq:cbound}\\
 \frac{1}{2\sigma_a}\frac{d}{dt}\norm{\hat A}_F^2
 &\le \frac14\norm{\frac{s}{\sigma_a}
 +\kappa\widehat{\mathcal C}\one_n}^2
 -\chi_a\norm{\hat A}_F^2.                         \label{eq:Abound}
\end{align}
Therefore, bounded stage cost and tracking signals imply bounded actor and
critic weights on the considered operating region. These estimates establish
boundedness but do not imply convergence of the weights, value function, or
control policy. The constraint analysis below does not require such
convergence.

\subsection{Applied Torque and Constraint Analysis}

The applied torque is written as
\begin{equation}
 u_{\rm ap}=u_{\rm cmd}+\Delta u,                  \label{eq:applied}
\end{equation}
where $\Delta u$ includes saturation, rate limits, sample-and-hold error,
and interface effects.  Define
\begin{equation}
 f=C\alpha+g+M\dot\alpha-d,\qquad
 r_c=T_a-f+\Delta u .                              \label{eq:residual}
\end{equation}

\begin{theorem}\label{thm:certificate}
Consider the closed-loop system under
(\ref{eq:alpha})--(\ref{eq:ucmd}), with $T_c>0$ and $\gamma$ defined by
(\ref{eq:gamma}).
Suppose $q_d$ and $k$ are twice continuously
differentiable and are bounded together with their derivatives up to second
order, with $k\geq k_{\min}>0$. Assume that
$\norm{r_c}\leq\bar r$ and
$\underline m I\preceq M\preceq\bar m I$ hold throughout the considered
closed-loop operation. Then the time-varying set defined by
$\norm{z(t)}<k(t)$ is forward invariant. Consequently,
\begin{equation}
\norm{e(t)}<k(t),\qquad t\geq T_c,                 \label{eq:finalbound}
\end{equation}
without requiring $\norm{e(0)}<k(0)$.
\end{theorem}

\begin{proof}
Consider the Lyapunov function candidate
\[
 V=B+\frac{1}{2}s\T Ms .
\]
Differentiating $B$ in (\ref{eq:blf}) and substituting
(\ref{eq:alpha}) gives
\begin{equation}
\dot B=-\frac{z\T K_1z}{\beta D}
+\frac{\gamma z\T s}{\beta D}+X-aX^2,\qquad
X=\frac{\gamma\dot\gamma\norm{e}^2}{\beta D}.     \label{eq:Bdot}
\end{equation}
Since $M\dot s+Cs=u_{\rm ap}-f$, the skew-symmetry property and
(\ref{eq:ucmd}), (\ref{eq:applied}), and (\ref{eq:residual}) yield
\begin{equation}
\dot V=-\frac{z\T K_1z}{\beta D}-s\T K_2s
+s\T r_c+X-aX^2.                                  \label{eq:identity}
\end{equation}
Define
\begin{equation}
\lambda=\min\left\{2\lambda_{\min}(K_1),
\frac{\lambda_{\min}(K_2)}{\bar m}\right\},\qquad
\eta=\frac{\bar r^2}{2\lambda_{\min}(K_2)}.        \label{eq:constants}
\end{equation}
Using $X-aX^2\leq1/(4a)$,
$s\T r_c\leq\lambda_{\min}(K_2)\norm{s}^2/2+\eta$, and
$\norm{z}^2/(\beta D)\geq2B$, we obtain
$\dot V\leq-\lambda V+\eta+\mathbf 1_{{t<T_c}}/(4a)$.
The comparison lemma then gives $V(t)\leq H(t)$, where
\begin{align}
H(t):={}&e^{-\lambda t}V(0)\nonumber\\[-1mm]
&\quad+\int_0^t e^{-\lambda(t-\tau)}
\left[\eta+\frac{\mathbf 1_{\{\tau<T_c\}}}{4a}\right]\,d\tau .
\label{eq:H}
\end{align}
These estimates hold on the maximal interval over which $D>0$. Since
$D(0)=k^2(0)>0$ and $H(t)$ is finite on every finite interval,
\[
\frac{D}{k^2}=e^{-2\beta B}
\geq e^{-2\beta V}
\geq e^{-2\beta H}.
\]
Consequently,
\begin{equation}
\frac{\norm{z(t)}}{k(t)}
\leq\sqrt{1-e^{-2\beta H(t)}}<1.                   \label{eq:zbound}
\end{equation}
The trajectory therefore cannot reach the boundary $\norm{z}=k$ in finite
time, and the set $\norm{z(t)}<k(t)$ is forward invariant. Finally,
$\gamma=1$ and $z=e$ for $t\geq T_c$, which proves
(\ref{eq:finalbound}).
\end{proof}

The residual bound in Theorem~\ref{thm:certificate} includes the effect of
actuator saturation and the control interface. A bound on $u_{\rm ap}$ alone
does not guarantee a bound on
$\Delta u=\sat(u_{\rm cmd})-u_{\rm cmd}$ because the raw barrier command can
grow as $D$ approaches zero. The theorem must therefore be applied on an
operating region where the assumed residual bound is valid. The numerical map in Section~\ref{sec:simulation} empirically examines which initial-error and transition-time combinations satisfy the final bound under the specified torque limits; it is not a numerical verification of the
residual assumption.

During the transition, the bound on $z$ should not be divided by $\gamma$
because $\gamma(0)=0$. Boundedness of the direct error can instead be examined
by defining $x=e/k$. Its upper-right Dini derivative satisfies
\begin{equation}
\mathrm{D}^{+}\norm{x}
\leq-\lambda_{\min}(K_1)\norm{x}+\frac{\norm{s}}{k}.
\label{eq:prebound}
\end{equation}
The term containing $\dot\gamma^2$ is radial and nonpositive in this
derivative. Together with the comparison bound on $s$,
(\ref{eq:prebound}) keeps the direct error finite while $\gamma$ increases
from zero.

\paragraph{Behavior after the transition}
For $t\geq T_c$, the transition-dependent term vanishes, and
(\ref{eq:H}) gives
\begin{align}
V(t)\leq{}&e^{-\lambda(t-T_c)}H(T_c)\nonumber\\[-1mm]
&+\frac{\eta}{\lambda}
\left[1-e^{-\lambda(t-T_c)}\right].
\label{eq:postH}
\end{align}
A nonzero residual therefore produces a bounded tracking neighborhood. If the
residual is zero, the bound decays. The corresponding asymptotic error ratio
satisfies
\begin{equation}
\limsup_{t\to\infty}\frac{\norm{e(t)}}{k(t)}
\leq\sqrt{1-\exp(-2\beta\eta/\lambda)}.             \label{eq:ultimate}
\end{equation}
Because the controller parameters also affect the realized residual, this
bound should not be used independently for gain selection.

\paragraph{Selecting the transition time}
The transition function in (\ref{eq:gamma}) satisfies
$\norm{\dot\gamma}_\infty=\pi/(2T_c)$. A shorter $T_c$ requires the robot to
reduce a given initial error more quickly and increases the
transition-dependent term in (\ref{eq:alpha}). The transition time must
therefore be selected together with the expected initial-error range and the
available torque and rate limits. The numerical map in
Fig.~\ref{fig:map} shows this relation for the tested settings and retains the
unsuccessful cases.

At each sample, the controller first computes $e$, $z$, $D$, $\alpha$, and
$s$, followed by the RBF features and $T_a$. The command in
(\ref{eq:ucmd}) is then sent to the actuator, and the applied or reported
input is used in the stage cost before updating the filter and learner
weights. The value of $D$ is checked before evaluating terms containing
$1/D$. If $D\leq0$ is detected, the simulation is stopped and the run is
classified as unsuccessful. No numerical lower bound or clipping is applied
to $D$.

\subsection{Parameter Selection for Robot Tasks}

The envelope should be chosen from the task before the feedback gains are
tuned.  If it represents tracking accuracy, $k$ can follow the tolerance of a
machining, inspection, or handover phase.  If it is also meant to imply joint
travel limits, the desired trajectory must retain componentwise clearance of
at least $k$.  The scalar Euclidean envelope is conservative for that purpose,
but it avoids switching between jointwise minima and keeps $\dot k$ well
defined.  A smooth conservative approximation is preferable when several
task margins have to be combined.

The transition time is then selected subject to the available actuator
limits. For a known range of task-switch errors, one can sweep $T_c$ using the
same torque and rate limits that will be applied online. The resulting map
does not prove feasibility between sampled points, but it identifies transition
times that are clearly too short. During commissioning, the barrier margin
$D/k^2$ and both the commanded and applied torques should be logged. A
decreasing domain margin together with sustained saturation indicates that
$T_c$ should be increased rather than modifying $D$ numerically.

$K_1$ sets the nominal contraction of the normalized position error, while
$K_2$ damps the velocity mismatch.  Raising either gain can shorten the
transient but also increases the raw command.  The coefficient $a$ appears in
the activation correction and in the proof term $1/(4a)$; changing it therefore
affects both transient input and the conservatism of the comparison bound.
The parameter $\beta$ scales the logarithmic storage and barrier feedback but
does not widen the admissible set.  These quantities should first be set with
the adaptive weights fixed at zero.  Actor and critic gains can then be
increased only while the base controller retains an adequate domain margin.
This ordering is useful because the theorem needs a bound on the realized
residual, not successful convergence of a learning algorithm.

The proof is continuous-time.  In a digital implementation, the feature
filter, learner, barrier calculation, and actuator interface are all sampled,
so intersample constraint preservation does not follow from
Theorem~\ref{thm:certificate}.  The step-size check below addresses numerical convergence of the
reported simulation, while a hardware certificate would additionally require
a sampled-data analysis or a measured intersample margin.

\section{Numerical Study}
\label{sec:simulation}

\subsection{Protocol}

We used the standard vertical-plane two-link dynamics in
(\ref{eq:robot}) with $m_1=\SimMassOne$ kg, $m_2=\SimMassTwo$ kg,
$l_1=\SimLinkOne$ m, $l_{c1}=\SimCOMOne$ m,
$l_{c2}=\SimCOMTwo$ m, and
$I_1=\SimInertiaOne$, $I_2=\SimInertiaTwo$ kg m$^2$.
The gravitational acceleration is $9.81$ m/s$^2$.  The plant includes viscous
friction coefficients $[0.8,\,0.6]\T$ N m s/rad, smoothed Coulomb-friction
levels $[0.9,\,0.7]\T$ N m with a $0.15$ rad/s transition scale, and bounded
sinusoidal disturbances with amplitudes $0.8$ and $0.5$ N m and angular
frequencies $1.7$ and $1.1$ rad/s, respectively.  The second disturbance has
a relative phase offset of $0.3$ rad.  Neither friction nor disturbance is
supplied to the proposed controller.  The reference trajectory is
$q_d=[0.35\sin(0.7t),\;0.55+0.25\cos(0.5t)]\T$.
% The supplied code contains the complete plant equations.

The fixed settings are $T=\SimDuration$ s, step \SimStepMs{} ms,
$T_c=\SimAcquisitionTime$ s,
$k(t)=\SimKZero+\SimKAmplitude\cos(\SimKFrequency t)$ rad, $\beta=\SimBeta$,
$K_1=\SimKOne$, $K_2=\SimKTwo I$, and $a=\SimActivationGain$.
Torque is bounded by
$|u_1|\le\SimTorqueLimitOne$ N m and
$|u_2|\le\SimTorqueLimitTwo$ N m.  The initial error is
$[\SimInitialErrorOne,\,\SimInitialErrorTwo]\T$ rad.  \SimRBFCount{}
deterministic
Gaussian centers plus a bias are used for each learner; all actor and critic weights start at zero, and the feature filter is initialized by $\bar\phi_c(0)=\phi_c(0)$.

Let $q_m$ denote the sampled position supplied to the controller and
$e_m=q_m-q_d$.  The actor and critic  inputs are
$x_a=[q_m\T/\pi,\,\dot q\T/4,\,e_m\T/k,\,s\T/4]\T,$
$ \ x_c=[e_m\T/k,\,s\T/4]\T,$ respectively.       
In the nominal study $q_m=q$; in the perturbed
cases the controller also evaluates $z,D,\alpha$, and the stage cost with
$e_m$, while performance is calculated from the true simulated error $e$.
The learning parameters are
$\sigma_a=\SimSigmaA$, $\sigma_c=\SimSigmaC$, $\chi_a=\SimChiA$,
$\chi_c=\SimChiC$, $\kappa=\SimKappa$, $\nu=\SimDiscount$,
$\tau_f=\SimFilterTime$, and common RBF width \SimRBFWidth.  The implemented stage cost assigns weights \SimCostError,
\SimCostVelocity, and \SimCostInput{} to $\norm{e_m}^2$, $\norm{s}^2$,
and $\norm{u_{\rm ap}}^2$, respectively. These settings are fixed before
the paired cases.

Three controllers share the same plant, reference, torque limits, and
initial conditions: the proposed controller, the same progressive-barrier
controller with actor and critic fixed at their initial zero weights, and a
saturated nominal
computed-torque controller.  An ordinary barrier controller is not included
because it is undefined at the deliberately infeasible initialization.

The computed-torque baseline uses the nominal model and the acceleration
command $v_c=\ddot q_d-2\zeta\omega(\dot q-\dot q_d)-\omega^2e_m$, with
$\omega=\SimCTFrequency~\mathrm{s}^{-1}$ and $\zeta=\SimCTDamping$.  Its
command $M_n(q_m)v_c+C_n(q_m,\dot q)\dot q+g_n(q_m)$ is subject to the same
componentwise torque clipping.

Table~\ref{tab:scenarios} lists the fixed perturbations used in the paired
comparison.

\begin{table*}[t]
\caption{Fixed Paired Perturbation Cases}
\label{tab:scenarios}
\centering
\scriptsize
\setlength{\tabcolsep}{5.0pt}
\begin{tabular}{ccccccc}
\toprule
Case & Link-2 mass increment (kg) & Friction scale & Disturbance scale &
Measurement amplitude (rad) & Phase offset (rad) & Reference-speed scale\\
\midrule
\SimulationScenarioRows
\bottomrule
\end{tabular}
\end{table*}

Performance is measured after $T_c$ using
\begin{align}
 \rho&=\norm{e}/k,\nonumber\\
 e_{\rm rms}&=\sqrt{\frac{1}{T-T_c}\int_{T_c}^{T}\norm{e}^2dt}. \label{eq:metrics}
\end{align}
A run succeeds only if it completes without leaving $D>0$ and satisfies
$\rho<1$ at every logged sample from $T_c$ onward.

\begin{figure}[t]
    \centering
    \includegraphics[width=0.98\columnwidth]{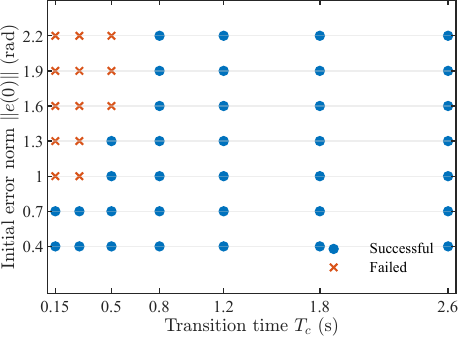}
    \caption{Feasibility under the same torque limits. Each marker represents
    one proposed-controller run, and failed settings are retained. Longer
    transition times accommodate larger initial errors in this grid.}
    \label{fig:map}
\end{figure}

\begin{figure*}[t]
    \centering
    \includegraphics[width=0.94\textwidth]{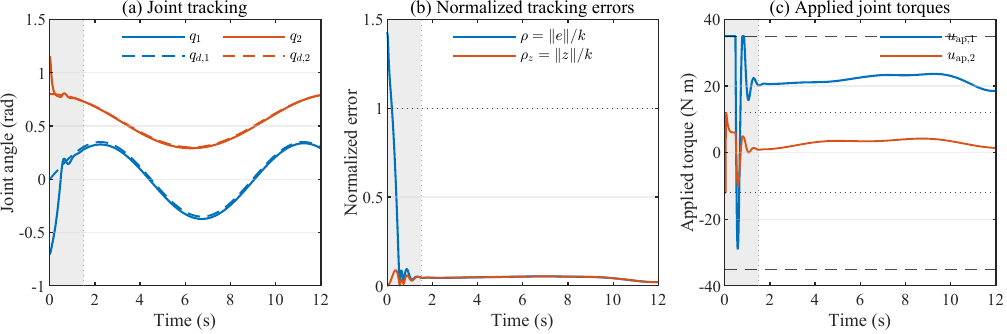}
   \caption{Nominal two-link response from an initial tracking error outside the
final bound: (a) joint tracking, (b) direct and scaled normalized tracking
errors, and (c) applied joint torques under the fixed limits. The vertical
dotted line marks $T_c$, and the shaded interval denotes the prescribed
transition. The horizontal line in panel (b) marks $\rho=1$, while those in
panel (c) mark the two joint torque limits.}
    \label{fig:acquisition}
\end{figure*}

\subsection{Transition and Perturbation Results}

Figure~\ref{fig:acquisition} shows that the transformed barrier state begins
inside its domain even though $\rho(0)=\SimInitialRatio>1$.  The direct error
last enters the final envelope at \SimProposedEntry{} s, well before
$T_c$, and its largest normalized error for $t\geq T_c$ is \SimProposedRho.  The torque limiter is active for \SimProposedSaturation\% of the nominal
samples.  This run therefore shows that the final bound can be reached under the
specified torque limits; it is not an unconditional consequence of the
theorem.

We next used \SimStressCaseCount{} fixed, paired cases spanning
\SimMassIncrementMin--\SimMassIncrementMax{} kg added at the link-2 center of
mass, friction scale \SimFrictionMin--\SimFrictionMax, disturbance scale
\SimDisturbanceMin--\SimDisturbanceMax, reference-speed scale
\SimSpeedMin--\SimSpeedMax, phase, and deterministic sinusoidal position-measurement
perturbations up to \SimNoiseMax{} rad.  The phase column in Table~\ref{tab:scenarios} gives the common phase offset applied to the disturbance and measurement signals; their fixed relative offsets remain
unchanged.  The controller uses the perturbed measurement in these cases, whereas all reported metrics use the
true simulated state; these noisy cases are empirical and are not part of the
continuous-time certificate.  No controller was retuned by case.
Figure~\ref{fig:comparison}
shows every successful-run value rather than only an average, and
Table~\ref{tab:comparison} gives the corresponding summaries.

\begin{figure*}[t]
    \centering
    \includegraphics[width=0.94\textwidth]{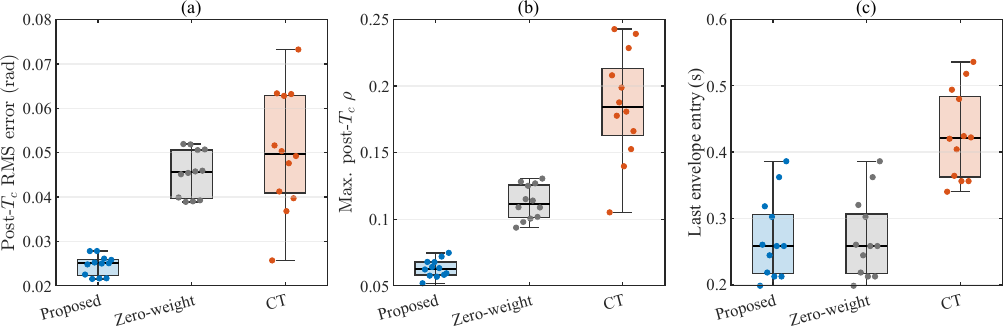}
    \caption{Paired results for the \SimStressCaseCount{} fixed perturbation
cases: (a) vector-error RMS for $t\geq T_c$, (b) largest normalized error
for $t\geq T_c$, and (c) last entry into the final envelope.  Markers
denote individual cases; boxes span the 25th--75th percentiles, center
lines denote medians, and whiskers show the full range.  CT denotes
computed torque; lower values are better.}
    \label{fig:comparison}
\end{figure*}

\begin{figure*}[!t]
    \centering
    \includegraphics[width=0.98\textwidth]{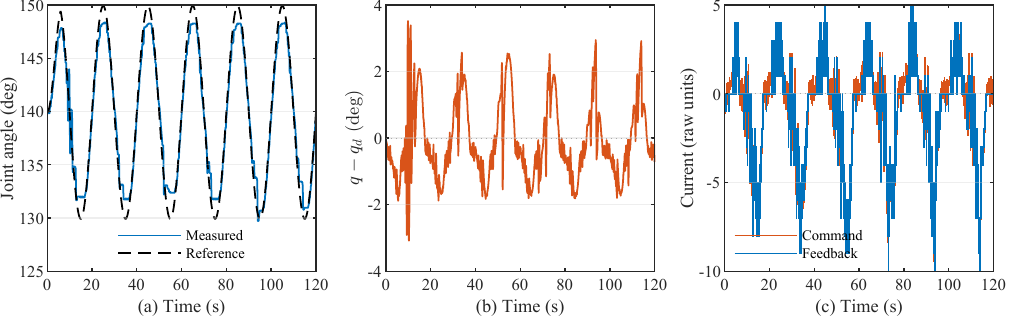}
    \caption{Complete single-joint physical record without smoothing:
    reference and measured joint angle, error recomputed with the paper's
    sign convention $q-q_d$, and commanded and measured current in logged raw
    units.}
    \label{fig:hardware}
\end{figure*}

\begin{table*}[t]
\caption{Two-Link Results Under Torque Limits}
\label{tab:comparison}
\centering
\small
\begin{tabular}{lcccc}
\toprule
Controller & Successful cases & Median $e_{\rm rms}$ (rad) & Worst $\rho$,
$t\ge T_c$ & Median last entry (s)\\
\midrule
Proposed & \SimProposedStressSuccess/\SimStressCaseCount &
\SimProposedStressMedianRMSE & \SimProposedStressWorstRho &
\SimProposedStressMedianEntry\\
Progressive barrier without learning & \SimFrozenStressSuccess/\SimStressCaseCount &
\SimFrozenStressMedianRMSE & \SimFrozenStressWorstRho &
\SimFrozenStressMedianEntry\\
Computed torque & \SimCTStressSuccess/\SimStressCaseCount &
\SimCTStressMedianRMSE & \SimCTStressWorstRho &
\SimCTStressMedianEntry\\
\bottomrule
\end{tabular}
\end{table*}

All three controllers succeed in this selected range. The proposed controller
gives the lowest median RMS error and the smallest worst-case normalized error
for $t\geq T_c$. Its median last-entry time is lower than that of computed
torque and equal to that of the zero-weight progressive-barrier controller.
Relative to this zero-weight controller, adaptation reduces the median RMS
error by \SimStressAdaptiveReduction\%. Reducing the integration step from
\SimStepMs{} to \SimFineStepMs{} ms changes nominal RMS error by only
\SimCoarseFineRMSEChange\%.

Figure~\ref{fig:map} varies $\norm{e(0)}$ over 0.4--2.2 rad and $T_c$ over
0.15--2.6 s along a fixed error direction.  The controller succeeds in
\SimMapSuccess{} of \SimMapTotal{} settings.  The row at 0.4 rad is initially
inside $k(0)$; among the \SimMapInfeasibleTotal{} initially infeasible grid
points, \SimMapInfeasibleSuccess{} succeed.  Failures concentrate at large
initial errors with short transition times, which is consistent with the
finite torque budget.  This boundary is empirical and depends on the plant,
reference, gains, and actuator limits.

\section{Experiment}
The archived record was obtained from Dynamixel ID~6 (Niryo joint J5) of a
Niryo Ned3 Pro manipulator using the current-control interface described
in \cite{zafaranchizadeh2026real}.  According to the experiment notes, the
scalar ($n=1$) controller tracked a $10^\circ$, 0.05-Hz sinusoid for
\ExpDuration{} s.  Because the controller source and current calibration were
not retained with the archived record, we report only quantities directly
supported by the log: measured and reference angle, signed error recomputed
as $q-q_d$, and commanded and measured current in raw actuator units.  The
median logging interval is \ExpLogPeriodMs{} ms; this is not asserted to be
the internal control period.

The experiment notes state that the standard joint controller was released
after pre-centering and that the wrist-side servo was placed in current-control
mode.  The interface included current scaling, direction-dependent bias
compensation, zero-crossing hysteresis, slew and current limits, and guards on
joint travel and barrier proximity.  These interface actions belong to
$\Delta u$ in (\ref{eq:applied}); they are not part of the paper-level control
law.

\begin{table}[t]
\caption{Single-Joint Tracking-Error Statistics (deg)}
\label{tab:hardware}
\centering
\small
\begin{tabular}{lrrrr}
\toprule
Interval & MAE & RMS & 95th pct. & Max.\\
\midrule
Complete record & \ExpFullMAE & \ExpFullRMSE & \ExpFullPercentile & \ExpFullMax\\
\ExpSteadyStart--\ExpDuration{} s & \ExpSteadyMAE & \ExpSteadyRMSE &
\ExpSteadyPercentile & \ExpSteadyMax\\
\bottomrule
\end{tabular}
\end{table}
The percentile and maximum columns are computed from absolute error.

Figure~\ref{fig:hardware} and Table~\ref{tab:hardware} use all
\ExpSampleCount{} logged samples.  The
\ExpSteadyStart--\ExpDuration{} s line separates the initial
portion but is not a repeated trial; its RMS error is \ExpSteadyRMSE$^\circ$.
Across \ExpWindowCount{} consecutive \ExpWindowLength-s windows of this single
run, RMS error ranges from
\ExpWindowRMSEMin$^\circ$ to \ExpWindowRMSEMax$^\circ$.  The recorded guard
and current-saturation fields contain \ExpGuardCount{} and
\ExpSaturationCount{} asserted samples, respectively, but those flags are
not treated as a constraint certificate because their thresholds cannot be
reconstructed from the archived data.

\section{Discussion and Conclusion}

The progressive scaling resolves the initialization problem of a logarithmic
barrier when the initial joint error lies outside the required tracking bound.
At the initial time, the scaled error is zero and the barrier is well defined.
The scaling then increases smoothly until it becomes one, after which the
barrier acts directly on the joint-position error. Under the stated residual
condition, the final tracking bound is satisfied after the prescribed
transition time.

The two-link simulations show both the capability and the practical limit of
this approach. All three controllers remain within the final bound at the
logged post-transition samples in the selected perturbation cases. Across
these cases, the proposed controller gives the lowest median RMS error, the
smallest worst-case normalized error, and a lower median last-entry time than
computed torque. Compared with the same progressive barrier controller
without learning, the actor--critic term reduces the median RMS error by
\SimStressAdaptiveReduction\%. The two progressive controllers have the same
median entry time, indicating that learning improves tracking accuracy and
the constraint margin rather than determining the transition speed.

The feasibility map shows why the transition time cannot be selected
independently of the initial error and actuator torque. A large initial error
combined with a short transition requires more torque than the actuator can
provide. Increasing the transition time allows a wider range of initial errors
to be handled. This relation is important when the controller is used together
with a motion planner: the planner specifies a collision-free reference, the
required joint-tracking tolerance, and the time by which that tolerance should
hold. Joint limits and collision avoidance before this time must still be
handled by the planner or an independent safety layer.

The physical record shows single-joint tracking using measured joint-position
and motor-current signals. Its full-record RMS error is
\ExpFullRMSE$^\circ$, decreasing to \ExpSteadyRMSE$^\circ$ after the initial
portion of the motion. This experiment supports implementation on physical
hardware, while the initially out-of-bound condition and coupled-joint
operation are evaluated in simulation. Future work will extend the experiment
to repeated multi-joint trials with calibrated torque, learner states, and
barrier variables recorded directly.

The results show that progressive barrier control provides a practical
closed-form method for robot tracking when the required error bound cannot be
satisfied initially. The method preserves a direct interpretation of the
joint-position bound, includes adaptive compensation for uncertain dynamics,
and makes the tradeoff between transition time, initial error, and available
torque explicit.

\section*{Acknowledgment}
\ifanonymous\else
This research was funded in whole, or in part, by the Luxembourg National
Research Fund (FNR), COSAMOS project, ref. CC22/IS/17432865/COSAMOS,
and the DOMINANT project, ref. CINTER/ANR/24/18972268/DOMINANT.
\fi
During the preparation of this manuscript, the authors used ChatGPT
(OpenAI) to assist with language editing of Abstract, Sections~I--VI, and figure captions; with \LaTeX{} formatting;
and with debugging of the simulation and plotting code. All technical
derivations, data, numerical results, references, and final wording were
reviewed and verified by the authors, who take full responsibility for
the content of the article.

 \IEEEtriggeratref{11}
\bibliographystyle{ieeetr}
\bibliography{references}

@book{spong2020robot,
  author    = {Spong, Mark W. and Hutchinson, Seth and Vidyasagar, M.},
  title     = {Robot Modeling and Control},
  edition   = {2nd},
  publisher = {Wiley},
  year      = {2020}
}

@article{tee2009barrier,
  author  = {Tee, Keng Peng and Ge, Shuzhi Sam and Tay, Eng Hock},
  title   = {Barrier {Lyapunov} Functions for the Control of Output-Constrained Nonlinear Systems},
  journal = {Automatica},
  volume  = {45},
  number  = {4},
  pages   = {918--927},
  year    = {2009},
  doi     = {10.1016/j.automatica.2008.11.017}
}

@article{tang2016adaptive,
  author  = {Tang, Zhong-Liang and Ge, Shuzhi Sam and Tee, Keng Peng and He, Wei},
  title   = {Adaptive Neural Control for an Uncertain Robotic Manipulator With Joint Space Constraints},
  journal = {International Journal of Control},
  volume  = {89},
  number  = {7},
  pages   = {1428--1446},
  year    = {2016},
  doi     = {10.1080/00207179.2015.1135351}
}

@article{rahimi2018joint,
  author  = {Rahimi, Hamed N. and Howard, Ian and Cui, Lei},
  title   = {Neural Adaptive Tracking Control for an Uncertain Robot Manipulator With Time-Varying Joint Space Constraints},
  journal = {Mechanical Systems and Signal Processing},
  volume  = {112},
  pages   = {44--60},
  year    = {2018},
  doi     = {10.1016/j.ymssp.2018.03.042}
}

@article{bechlioulis2008prescribed,
  author  = {Bechlioulis, Charalampos P. and Rovithakis, George A.},
  title   = {Robust Adaptive Control of Feedback Linearizable {MIMO} Nonlinear Systems With Prescribed Performance},
  journal = {IEEE Transactions on Automatic Control},
  volume  = {53},
  number  = {9},
  pages   = {2090--2099},
  year    = {2008},
  doi     = {10.1109/TAC.2008.929402}
}

@article{mayne2000mpc,
  author  = {Mayne, David Q. and Rawlings, James B. and Rao, Christopher V. and Scokaert, Pierre O. M.},
  title   = {Constrained Model Predictive Control: Stability and Optimality},
  journal = {Automatica},
  volume  = {36},
  number  = {6},
  pages   = {789--814},
  year    = {2000},
  doi     = {10.1016/S0005-1098(99)00214-9}
}

@article{song2018deferred,
  author  = {Song, Yong-Duan and Zhou, Shuyan},
  title   = {Tracking Control of Uncertain Nonlinear Systems With Deferred Asymmetric Time-Varying Full State Constraints},
  journal = {Automatica},
  volume  = {98},
  pages   = {314--322},
  year    = {2018},
  doi     = {10.1016/j.automatica.2018.09.032}
}

@article{zhao2019feasibility,
  author  = {Zhao, Kai and Song, Yongduan},
  title   = {Removing the Feasibility Conditions Imposed on Tracking Control Designs for State-Constrained Strict-Feedback Systems},
  journal = {IEEE Transactions on Automatic Control},
  volume  = {64},
  number  = {3},
  pages   = {1265--1272},
  year    = {2019},
  doi     = {10.1109/TAC.2018.2845707}
}

@article{zhao2022deferred,
  author  = {Zhao, Kai and Chen, Long and Chen, C. L. Philip},
  title   = {Event-Based Adaptive Neural Control of Nonlinear Systems With Deferred Constraint},
  journal = {IEEE Transactions on Systems, Man, and Cybernetics: Systems},
  volume  = {52},
  number  = {10},
  pages   = {6273--6282},
  year    = {2022},
  doi     = {10.1109/TSMC.2022.3143359}
}

@article{kong2023global,
  author  = {Kong, Linghuan and He, Wei and Liu, Zhijie and Yu, Xinbo and Silvestre, Carlos},
  title   = {Adaptive Tracking Control With Global Performance for Output-Constrained {MIMO} Nonlinear Systems},
  journal = {IEEE Transactions on Automatic Control},
  volume  = {68},
  number  = {6},
  pages   = {3760--3767},
  year    = {2023},
  doi     = {10.1109/TAC.2022.3201258}
}

@article{zhang2023deferredrobot,
  author  = {Zhang, Yu and Kong, Linghuan and Zhang, Shuang and Yu, Xinbo and Liu, Yu},
  title   = {Improved Sliding Mode Control for a Robotic Manipulator With Input Deadzone and Deferred Constraint},
  journal = {IEEE Transactions on Systems, Man, and Cybernetics: Systems},
  volume  = {53},
  number  = {12},
  pages   = {7814--7826},
  year    = {2023},
  doi     = {10.1109/TSMC.2023.3301662}
}

@article{ames2017cbf,
  author  = {Ames, Aaron D. and Xu, Xiangru and Grizzle, Jessy W. and Tabuada, Paulo},
  title   = {Control Barrier Function Based Quadratic Programs for Safety Critical Systems},
  journal = {IEEE Transactions on Automatic Control},
  volume  = {62},
  number  = {8},
  pages   = {3861--3876},
  year    = {2017},
  doi     = {10.1109/TAC.2016.2638961}
}

@inproceedings{rauscher2016robot,
  author    = {Rauscher, Manuel and Kimmel, Melanie and Hirche, Sandra},
  title     = {Constrained Robot Control Using Control Barrier Functions},
  booktitle = {Proceedings of the IEEE/RSJ International Conference on Intelligent Robots and Systems},
  pages     = {279--285},
  year      = {2016},
  doi       = {10.1109/IROS.2016.7759067}
}

@article{breeden2022sampled,
  author  = {Breeden, Joseph and Garg, Kunal and Panagou, Dimitra},
  title   = {Control Barrier Functions in Sampled-Data Systems},
  journal = {IEEE Control Systems Letters},
  volume  = {6},
  pages   = {367--372},
  year    = {2022},
  doi     = {10.1109/LCSYS.2021.3076127}
}

@article{garg2024advances,
  author  = {Garg, Kunal and Usevitch, James and Breeden, Joseph and Black, Mitchell and Agrawal, Devansh and Parwana, Harshit and Panagou, Dimitra},
  title   = {Advances in the Theory of Control Barrier Functions: Addressing Practical Challenges in Safe Control Synthesis for Autonomous and Robotic Systems},
  journal = {Annual Reviews in Control},
  volume  = {57},
  pages   = {100945},
  year    = {2024},
  doi     = {10.1016/j.arcontrol.2024.100945}
}

@article{lopez2021adaptive,
  author  = {Lopez, Brett T. and Slotine, Jean-Jacques E. and How, Jonathan P.},
  title   = {Robust Adaptive Control Barrier Functions: An Adaptive and Data-Driven Approach to Safety},
  journal = {IEEE Control Systems Letters},
  volume  = {5},
  number  = {3},
  pages   = {1031--1036},
  year    = {2021},
  doi     = {10.1109/LCSYS.2020.3005923}
}

@article{helwa2019robust,
  author  = {Helwa, Mohamed K. and Heins, Adam and Schoellig, Angela P.},
  title   = {Provably Robust Learning-Based Approach for High-Accuracy Tracking Control of {Lagrangian} Systems},
  journal = {IEEE Robotics and Automation Letters},
  volume  = {4},
  number  = {2},
  pages   = {1587--1594},
  year    = {2019},
  doi     = {10.1109/LRA.2019.2896728}
}

@article{brunke2022safe,
  author  = {Brunke, Lukas and Greeff, Melissa and Hall, Adam W. and Yuan, Zhaocong and Zhou, Siqi and Panerati, Jacopo and Schoellig, Angela P.},
  title   = {Safe Learning in Robotics: From Learning-Based Control to Safe Reinforcement Learning},
  journal = {Annual Review of Control, Robotics, and Autonomous Systems},
  volume  = {5},
  pages   = {411--444},
  year    = {2022},
  doi     = {10.1146/annurev-control-042920-020211}
}

@article{vamvoudakis2010actor,
  author  = {Vamvoudakis, Kyriakos G. and Lewis, Frank L.},
  title   = {Online Actor--Critic Algorithm to Solve the Continuous-Time Infinite Horizon Optimal Control Problem},
  journal = {Automatica},
  volume  = {46},
  number  = {5},
  pages   = {878--888},
  year    = {2010},
  doi     = {10.1016/j.automatica.2010.02.018}
}

@article{nohooji2024actor,
  author  = {Rahimi Nohooji, Hamed and Zaraki, Abolfazl and Voos, Holger},
  title   = {Actor--Critic Learning Based {PID} Control for Robotic Manipulators},
  journal = {Applied Soft Computing},
  volume  = {151},
  pages   = {111153},
  year    = {2024},
  doi     = {10.1016/j.asoc.2023.111153}
}

@article{zafaranchizadeh2026real,
title = {Real-Hardware Deployment of a {Nussbaum}-Function {PID}
Controller on a Current-Controlled Low-Cost Actuator via
Hardware-Aware {Optuna} Tuning},
  author={Zafaranchizadeh Moghaddam, Danial and Tveretina, Olga and Zaraki, Abolfazl},
  journal={Sensors},
  volume={26},
  number={13},
  pages={4212},
  year={2026},
  publisher={MDPI}
}

\end{document}